\documentclass[10pt, conference, compsocconf]{IEEEtran}
\usepackage{amsmath,amscd,amssymb,graphics}
\usepackage{mathrsfs} 

\newtheorem{theorem}{Theorem}[section]
\newtheorem{corollary}[theorem]{Corollary} 
\newtheorem{lemma}[theorem]{Lemma}

 \newtheorem{definition}[theorem]{Definition}
\newtheorem{remark}[theorem]{Remark}
\newcommand{\abs}[1]{\lvert#1\rvert}

\def\E{{\mathbb{E}}}

\def\R{{\mathbb R}}
\def\e{\varepsilon}

\begin{document}
\title{Predictive PAC learnability: a paradigm for learning from exchangeable input data}


\author{\IEEEauthorblockN{Vladimir Pestov}
\IEEEauthorblockA{Department of Mathematics and Statistics\\
University of Ottawa\\
Ottawa, Ontario, Canada\\
vpest283@uottawa.ca}
}

\maketitle

\begin{abstract}
Exchangeable random variables form an important and well-studied generalization of i.i.d. variables, however simple examples show that no nontrivial concept or function classes are PAC learnable under general exchangeable data inputs $X_1,X_2,\ldots$. Inspired by the work of Berti and Rigo on a Glivenko--Cantelli theorem for exchangeable inputs, we propose a new paradigm, adequate for learning from exchangeable data: predictive PAC learnability. A learning rule $\mathcal L$ for a function class $\mathscr F$ is predictive PAC if for every $\e,\delta>0$ and each function $f\in {\mathscr F}$, whenever $\abs{\sigma}\geq s(\delta,\e)$, we have with confidence $1-\delta$ that the expected difference between $f(X_{n+1})$ and the image of $f\vert\sigma$ under $\mathcal L$ does not exceed $\e$ conditionally on $X_1,X_2,\ldots,X_n$. Thus, instead of learning the function $f$ as such, we are learning to a given accuracy $\e$ the predictive behaviour of $f$ at the future points $X_i(\omega)$, $i>n$ of the sample path. Using de Finetti's theorem, we show that if a universally separable function class $\mathscr F$ is distribution-free PAC learnable under i.i.d. inputs, then it is distribution-free predictive PAC learnable under exchangeable inputs, with a slightly worse sample complexity. 
\end{abstract}


% \fi
\IEEEpeerreviewmaketitle

\section{Introduction}

In the classical theory of statistical learning as initiated in \cite{VC1971,BEHW} (see \cite{vapnik02} for a historical and philosophical perspective) data inputs are traditionally modelled by a sequence of i.i.d. random variables $(X_i)$. Generalizating this approach usually involves easing the i.i.d. restriction on the sequence of inputs, all the while trying to obtain the same conclusions as in the classical theory, namely the uniform convergence of empirical means and subsequently the PAC learnability of a concept or a function class under the usual combinatorial restrictions in terms of shattering. For instance, the i.i.d. condition can be relaxed to that of being an ergodic stationary sequence (\cite{pollard}, p. 9), or a $\beta$-mixing sequence \cite{vidyasagar}. As to $\alpha$-mixing sequences, they are known to result in the same PAC learnable function classes under a single distribution \cite{vidyasagar05}, although it is still unknown whether uniform convergence of empirical means takes place \cite{yu}. An interesting recent investigation is \cite{kontorovich}.

However, at some point this approach hits a wall. Among the best studied classes of dependent stationary random variables are exchangeable random variables  \cite{dF37}; \cite{billingsley}, p. 473; \cite{kallenberg}, \cite{kingman}. A sequence of r.v. $(X_i)$ is {\em exchangeable,} if for every finite sequence $(i_1,i_2,\ldots,i_n)$ of integers the joint distributions of $(X_{i_1},X_{i_2},\ldots,X_{i_n})$ and of $(X_1,X_2,\ldots,X_n)$ are the same. According to the famous De Finetti theorem \cite{dF37,HS}, a sequence $(X_i)$ is exchangeable if and only if the joint distribution $P$ on $\Omega^{\infty}$ is a mixture of product distributions (that is, $(X_i)$ is a mixture of a family of i.i.d. random sequences). 

A nice illustration and 
the most extreme example of an exchangeable sequence which is not i.i.d is a sequence of identical copies of one and the same random variable, $X_i=X$, $i=1,2,\ldots$. The joint distribution of this process is a measure supported on the diagonal of the infinite product space $\Omega^{\infty}$, which is clearly a mixture of infinite powers of all Dirac point masses on $\Omega$. 

Now, it is immediately clear that no nontrivial function class $\mathscr F$ on a domain $\Omega$ will be PAC learnable under such a data input process: almost every sample path $\bar x$ will be constant, $\bar x = (x,x,x,\ldots)$, thus revealing no information about the values of a function $f\in{\mathscr F}$ away from $x$. Consequently, if we want to be able to learn from exchangeable data inputs, the paradigm of learnability itself has to be re-examined.

A way out was shown by Berti and Rigo in their visionary note \cite{BR} where they prove that the classical Glivenko--Cantelli theorem holds for a sequence $(X_i)$ of exchangeable random variables if and only if the sequence is i.i.d. 
At the same time, they observe that the classical GC theorem is formally equivalent to the statement about the predictive distribution being approximated by the observed frequency:
\begin{eqnarray*}
% P\left\{\omega\colon 
\sup_{t}\vert F_n(t,\omega)-
P(X_{n+1}\leq t\Vert  
X_1,\ldots,X_n)(\omega)\vert\to 0\mbox{ a.s.}
% \right\}=1.
\end{eqnarray*}
Here $F_n(t,\omega) =(1/n)\sum_{i=1}^n I_{(-\infty,t]}(X_i)$ is the empirical mean of the indicator function, and $P(\cdot\Vert X_1,\ldots,X_n)$ is the conditional probability.
As shown in \cite{BR}, in this form the statement remains valid if the r.v. $(X_i)$ are exchangeable, and the result can be considered as a conditional (or: predictive) version of the classical Glivenko-Cantelli theorem.

Since the uniform Glivenko-Cantelli theorems are at the heart of statistical learning, one would think that the approach of Berti and Rigo should have consequences for learning from exchangeable inputs. We show that this is indeed the case: by replacing PAC learnability with {\em predictive} PAC learnability, one arrives at a new broad paradigm of learnability suited for learning under exchangeable inputs. 

Say that a function class $\mathscr F$ is {\em predictively PAC learnable} under a given class $\mathcal P$ of exchangeable random processes $(X_n)$ if there exists a {\em predictive PAC learning rule} for $\mathscr F$ under $\mathcal P$, that is, a map $\mathcal L$ from the sample space $\mathcal S$ to a hypothesis class $\mathscr H$ such that 
\begin{eqnarray*}
P\{\sigma\colon 
\E\left(\vert({\mathcal L}(f\vert_\sigma)- f)(X_{n+1})\vert\,
\Vert X_1,X_2,\ldots, X_n\right)>\e \} \to 0
\end{eqnarray*}
uniformly in $f\in {\mathscr F}$ and $(X_i)\in{\mathcal P}$. This is different from PAC learnability in that the expected value of $\vert{\mathcal L}(f\vert_\sigma)- f\vert$ is replaced with the conditional expectation given $X_1,X_2,\ldots,X_n$. If in particular $(X_i)$ are i.i.d., the above definition is a reformulation of PAC learnability under the family of corresponding laws on the domain $\Omega$.

% which satisfies
% 
% \begin{eqnarray*}
% \sup_{(X_n)\in {\mathcal P}}\sup_{f\in {\mathscr F}}
% P\{\sigma\in\Omega^n\colon 
% \E(\vert{\mathcal L}(f\vert_\sigma)- f\vert &&\\
% \Vert X_1,X_2,\ldots, X_n)>\e \} \to 0&\mbox{as}&n\to\infty.
% \end{eqnarray*}

% If $\mathscr F$ is predictively PAC learnable under the class of {\em all} exchangeable sequences of random variables, we say that $\mathscr F$ is distribution-free conditionally PAC learnable, or just conditionally PAC learnable.

We show that if a function class $\mathscr F$ is distribution-free PAC learnable under the usual assumption that the data sample inputs are i.i.d., then $\mathscr F$ is predictively PAC learnable under the class of all sequences of exchangeable data inputs. Our results are obtained under the assumption that $\mathscr F$ is universally separable.

% A similar result holds for learning under a fixed sequence of exchangeable random variables versus PAC learnability under the corresponding marginal distribution. 

% The above paradigm might be well applicable to great many situations. For instance, as the Ottawa Heart Institute is trying to learn which genes are responsible for hereditary heart disease from a sample of patients, it is very reasonable to assume that the data is exchangeable, yet dependent, inflenced by the climatic and lifestyle setting of the province of Ontario, as well as a particular genetic mix-up of the Ontario population. For this reason the findings will be applicable to future patients treated at OHI, but perhaps not elsewhere. For the same reason, one should resist the temptation to enlarge the training set by using data from other clinical centres, because this approach is based on the assumption that the data is i.i.d.
% 
% As noted by Berti and Rigo, the conditional approach to the Glivenko-Cantelli theorem sits well with the predictive approach to statistics advocated by de Finetti and numerous other statisticians, and the same can be said about the conditional approach to PAC learnability. 

% This approach, in its turn, has its limits, and we discuss some of them, as well as future directions of research, in the Conclusion.

\section{Setting for learnability}

Here we review the PAC learnability model \cite{AB,BEHW,vapnik98,vidyasagar} in order to fix a precise setting.
The {\em domain}, or {\em instance space}, $\Omega=(\Omega,{\mathscr A})$ is a {\em measurable space,} that is, a set $\Omega$ equipped with a sigma-algebra of subsets $\mathscr A$. We will assume that $\Omega$ is a {\em standard Borel space,} that is, a complete separable metric space equipped with the sigma-algebra of Borel subsets. For intstance, without loss in generality one can always assume that $\Omega=\R^k$ is the Euclidean space.

Denote by ${\mathscr B}(\Omega,[0,1])$ the collection of all Borel measurable functions from $\Omega$ to $[0,1]$.
A {\em function class} $\mathscr F$ is a subfamily of ${\mathscr B}(\Omega,[0,1])$. 

% One needs to impose on $\mathscr F$ additional conditions of measurability. We will assume $\mathscr F$ to be {\em image admissible Souslin} (\cite{dudley}, pages 186--187) if there are a Polish space $I$ and a surjective map $f\colon I\to {\mathscr F}$ such that the mapping
% \[I\times \Omega\ni (i,\omega)\mapsto f_i(\omega)\in\to [0,1]\]
% is measurable with regard to the Borel product structure on $I\times \Omega$.

The family $P(\Omega)$ of all probability measures on $(\Omega,{\mathscr A})$ is itself a measurable space, whose sigma-algebra is generated by the functions $\nu\mapsto \nu(A)$ from $P(\Omega)$ to $\R$, as $A$ runs over ${\mathscr A}$.

In the PAC learning model, a set $\mathcal P$ of probability measures on $\Omega$ is fixed. Usually either ${\mathcal P}=P(\Omega)$ is the set of all probability measures ({\em distribution-free learning}), or ${\mathcal P}=\{\mu\}$ is a single measure ({\em learning under a fixed distribution}). 

% In our article, ${\mathcal P}=P_{na}(\Omega)$ is the family of all non-atomic measures.
% 
% Every probability measure $\mu$ on $\Omega$ defines a distance $d_{\mu}$ on $\mathscr A$ as follows:
% \[d_{\mu}(A,B)=\mu\left(A\bigtriangleup B \right).\] 
% Often it is convenient to approximate the concepts from $\mathscr C$ with elements of the {\em hypothesis space,} $\mathscr H$, which is, technically, a subfamily of $\mathscr A$ whose closure with regard to each (pseudo)metric $d_{\mu}$, $\mu\in{\mathcal P}$, contains $\mathscr C$. In our article, $\mathscr H$ will always coincide with $\mathscr C$.

% The {\em VC-dimension} of a family $\mathscr C$ of subsets of $\Omega$ is the supremum of all natural numbers $n$ for which there exists a subset $A\subseteq\Omega$ {\em shattered} by $\mathscr C$ in the sense that every subset $B\subseteq A$ is of the form $B=A\cap C$ for a suitable $C\in {\mathscr C}$. (Cf. Figure \ref{fig:shattered}.) 
% We will denote this combinatorial parameter $\VC({\mathscr C})$ or, more precisely, $\VC({\mathscr C}\upharpoonright\Omega)$. 
% 
% \begin{figure}[ht]
% \begin{center}
% \scalebox{0.225}[0.225]{\includegraphics{shattered.eps}} 
% \caption{A subset $A\subseteq\Omega$ is shattered by $\mathscr C$.}
% \label{fig:shattered}
% \end{center}
% \end{figure}

%If $\mu$ is a probability measure on $(\Omega,{\mathscr A})$, we denote $d_{\mu}$ the corresponding uniform metric on $\mathscr A$:
%\[d_{\mu}(A,B)=\mu\left(A\bigtriangleup B \right).\]

A {\em learning sample} is a pair $s$ consisting of a finite subset $\sigma$ of $\Omega$ and of a function on $\sigma$. It is convenient to assume that elements $x_1,x_2,\ldots,x_n\in\sigma$ are ordered, and thus the set of all samples $(\sigma,\tau)$ with $\abs\sigma=n$ can be identified with $\left(\Omega\times[0,1]\right)^n$. For $\sigma\in\Omega^n$ and a function $f\in {\mathscr F}$ we will denote $f\upharpoonright\sigma$ the sample obtained by restricting $f$ to $\sigma$.

A {\em learning rule} is a mapping 
\[{\mathcal L}\colon \bigcup_{n=1}^\infty\Omega^n\times[0,1]^n\to {\mathscr B}(\Omega,[0,1]),\]
which is measurable with regard to every Borel structure induced on ${\mathscr B}(\Omega,[0,1])$ by the distances $L^1(\mu)$, $\mu\in {\mathcal P}$.

A learning rule $\mathcal L$ is {\em consistent} if for every $f\in {\mathscr F}$ and each $\sigma\in\Omega^n$ one has 
\[{\mathcal L}(f\upharpoonright \sigma)\upharpoonright\sigma = f\upharpoonright\sigma.\]
Consistent learning rules exist for every function class $\mathscr F$ under mild measurability restrictions.

%satisfying a slight measurability restriction: for every $\mu\in P(\Omega)$, $T\in {\mathscr C}$ and $n\in\N$, the function
%\[\Omega^n\times[0,1]^n\ni (\sigma,s)\mapsto \mu\left({\mathcal L}(\sigma,s)\bigtriangleup T \right)\in\R\]
%is measurable. 
% Let $T\in {\mathscr C}$ (a {\em target concept}) and let $P\in P(\Omega)$ be a probability measure on $\Omega$. 

A learning rule $\mathcal L$ is {\em probably approximately correct} ({\em PAC}) for the function class $\mathscr F$ {\em under the class of measures ${\mathcal P}$} if for every $\e>0$
\[\sup_{\mu\in {\mathcal P}}\sup_{f\in {\mathscr F}}
P\left\{\sigma\in\Omega^n\colon \E_\mu\vert{\mathcal L(f\upharpoonright\sigma)}- f\vert\vert>\e \right\} \to 0\]
as $n\to\infty$. Here $P$ stands for $\mu^{\otimes n}$.

Equivalently, there is a function $s(\e,\delta)$ ({\em sample complexity} of $\mathcal L$) such that for each $f\in{\mathscr F}$ and every $\mu\in {\mathcal P}$ an i.i.d. sample $\sigma$ with $\geq s(\e,\delta)$ points has the property $\E_\mu\vert f-{\mathcal L}(f\upharpoonright \sigma)\vert<\e$ with confidence $\geq 1-\delta$.

A function class $\mathscr F$ is {\em PAC learnable} {\em under $\mathcal P$}, if there exists a PAC learning rule for $\mathscr F$ under $\mathcal P$. 

If $\mathcal P = P(\Omega)$ is the set of all probability measures, then $\mathscr F$ is said to be (distribution-free) {\em PAC learnable}. At the same time, learnability under intermediate families of measures on $\Omega$ has received considerable attention, cf. Chapter 7 in \cite{vidyasagar}.

A closely related concept to that of a PAC learnable class is that of a {\em uniform Glivenko--Cantelli} function class, that is, a function class $\mathscr F$ such that for each $\delta,\e>0$ one has, whenever $n\geq s(\delta,\e)$,
\[\sup_{\mu\in P(\Omega)}P\left\{\sup_{f\in{\mathscr F}}\left\vert \E_{\mu}(f) - \frac 1nS_n(f) \right\vert\geq \e\right\}<\delta.
\]
One also says that $\mathscr C$ has the property of {\em uniform convergence of empirical means} ({\em UCEM} property). Here $s(\delta,\e)$ is the {\em sample complexity} of the uniform Glivenko-Cantelli class (which in general has to be distinguished from the sample complexity of a learning rule).

% \begin{equation}
% \label{eq:glivenko}
% \lim_{n\to\infty}
% \sup_{\mu\in P(\Omega)}P\left\{\sup_{f\in{\mathscr F}}\left\vert \E_{\mu}(f) - S_n(f) \right\vert\geq \e\right\}\to 0.
% \end{equation}
% 
% One also says that $\mathscr C$ has the property of {\em uniform convergence of empirical means} ({\em UCEM} property). The notion of sample complexity of a uniform Glivenko-Cantelli class $\mathscr F$ is introduced in the same way as for learnability, by the condition
% 
% whenever $n\geq s(\delta,\e)$.

Every uniform Glivenko--Cantelli function class is PAC learnable, for instance, every consistent learning rule for $\mathscr F$ is PAC, with the same learning sample complexity. For concept classes, the converse is also true, though not for function classes in general. 

A function class $\mathscr F$ is {\em universally separable} \cite{pollard} if it contains a countable subfamily ${\mathscr F}^\prime$ with the property that every $f\in {\mathscr F}$ is a pointwise limit of a sequence $(f_n)$ of functions from ${\mathscr F}^\prime$: for each $x\in\Omega$, one has $f_n(x)\to f(x)$ as $n\to\infty$.

Notice that in this paper, we only talk of {\em potential} learnability, adopting a purely information-theoretic viewpoint. 

\section{Exchangeable variables and de Finetti's theorem}

De Finetti's theorem, in its classical form (\cite{dF37}, Ch. IV; \cite{HS}, Th. 7.2) states that a sequence $(X_i)$ of random variables taking values in a standard Borel space $\Omega$ is exchangeable if and only if the joint distribution $P$ of the sequence is a mixture of i.i.d. distributions. More precisely, there exists a probability measure $\eta$ on the Borel space $P(\Omega)$ of probability measures on $\Omega$ (the directing measure) so that 
\begin{equation}
\label{eq:dec}
P = \int_{P(\Omega)}\theta^{\infty}\,\eta(d\theta),\end{equation}
in the sense that for every measurable function $f$ on $\Omega^\infty$ one has
\[\E(f) =\int \E_{\theta^\infty}(f)\,\eta(d\theta).\]
In this spirit, $\theta$ will denote a (random) element of $P(\Omega)$, and ``almost all $\theta$'' is to be understood in the sense of directing measure $\eta$.

A slightly different viewpoint, adopted in \cite{kallenberg}, is to fix a {\em random measure} $\nu$, that is, a measurable mapping from the basic probability space to $P(\Omega)$. Under this approach, de Finetti's theorem can be put in the following, essentially equivalent, form. Denote by $\mathscr T$ the tail sigma-field on $\Omega^\infty$. Then, conditionally on $\mathscr T$, the sequence $(X_i)$ is i.i.d.:
\[P(\omega\in\cdot\Vert {\mathscr T}) = \nu^{\infty}\mbox{ a.s.}\]

Note that if $\theta\neq\zeta$, then $\theta^\infty$ and $\zeta^\infty$ are mutually singular. This follows from a remark of Kakutani \cite{kakutani}, p. 223: fix $f$ with $\E_\theta(f)\neq\E_\zeta(f)$, then the empirical mean
\[\frac 1 n S_n(f)= \frac 1n \sum_{i=1}^n f(X_i)
\]converges at the same time $\theta^\infty$-a.s. to $\E_\theta(f)$ and $\zeta^\infty$-a.s. to $\E_\zeta(f)$. 
% A study of the respective domains of convergence allows to conclude that $\theta^\infty,\zeta^\infty$ are mutually singular. 
This observation helps to understand the decomposition (\ref{eq:dec}).

The strong law of large numbers for exchangeable variables (cf. e.g. \cite{kingman}, Eq. (2.2) on p. 185, also \cite{kallenberg}, Proposition 1.4(i)), says that 
\begin{equation}
\label{eq:slln}
\frac 1 n S_n(f)\to \E(f\Vert \mathscr T)\end{equation}
almost surely. If $P(A)=1$, then a.s. $\nu(A)=1$, that is, for almost all $\theta$, one has $\theta(A)=1$. Thus, the convergence in (\ref{eq:slln}) takes place $\theta$-a.s. for almost all $\theta\in\Theta$. One concludes:
\begin{equation}
\label{eq:constant}
\mbox{For a.e. }\theta,~~
\E(f(X_1)\vert {\mathcal T}) = \E_\theta(f)~~\theta\mbox{ a.s.}
\end{equation}
Informally, the conditional expectation $\E(f(X_1)\vert {\mathcal T})$ given the tail sigma-field is viewed by almost every non-random measure $\theta$ as a constant function, identically assuming the value $\E_{\theta}(f)$.

% Of course, different measures may perceive different constants.
%
% 
% \begin{lemma}
% Let $\mathscr F$ be a universally separable family of functions from $\Omega$ to $[0,1]$, and let $(X_i)$ be a sequence of exchangeable $\Omega$-valued random variables with a directing measure $\eta$ on $P(\Omega)$. Denote by $\mathscr T$ the tail sigma-field. Then for a.e. $\theta\in P(\Omega)$ one has
% \[\left(\forall f\in {\mathscr F}~~\E(f\Vert {\mathscr T}) = \E_\theta(f)\right)\mbox{ $\theta$-a.s.}\]
% \end{lemma}
% 
% \begin{proof}
% Fix a countable subfamily ${\mathscr F}^\prime$ as in the definition of universal separability, and set
% \[A = \bigcap_{f\in {\mathscr F}^\prime}\{\omega\colon \E(f\Vert {\mathscr T}) = \E_\theta(f)\}.\]
% This $A$ is a measurable subset with $\theta(A)=1$. Let now $f\in {\mathscr F}$. Fix a sequence $(f_n)$ of functions from ${\mathscr F}^\prime$ that converges pointiwise to $f$. By the Lebesgue dominated convergence theorem, 
% \[\E_\theta(f_n)\to \E_\theta(f).\]
% Fix $\omega\in A$. One has $\E(f_n\Vert {\mathscr T})(\omega)=\E_{\theta}(f_n)$ for all $n$.
% \end{proof}

\begin{lemma} 
\label{l:tail}
Let $X_1,X_2,\ldots$ be a sequence of exchangeable random variables taking values in a standard Borel space $\Omega$. Then for every measurable function $f$ on $\Omega$, for all $i$ and all $j>n$:
\[\E\left(\E(f(X_i)\Vert {\mathscr T})\Vert X_1,\ldots,X_n\right) = \E(f(X_{j})\Vert X_1,\ldots,X_n)\]
a.s., where $\mathscr T$ is the tail sigma-field. Consequently, if $\mathscr G$ is a countable family of measurable functions, then one has
\begin{eqnarray*}
\forall f\in{\mathscr G}~~\E\left(\E(f(X_i)\Vert {\mathscr T})\Vert X_1,\ldots,X_n\right)\phantom{xxxxxxxxxx}\\
= \E(f(X_{j})\Vert X_1,\ldots,X_n)\end{eqnarray*}
almost surely.
\end{lemma}

\begin{proof}
Because of exchangeability, one can assume without loss in generality that $i=1$ and $j=n+1$. Now
it is enough to establish the result for indicator functions $f=I_A$ of some generating family of Borel subsets $A\subseteq\Omega$, for instance, by identifying $\Omega$ with $\R$ and considering the intervals $A=(-\infty,t]$. In this form, the result has been proved in Berti and Rigo \cite{BR}, where a stronger assertion appears as formula (7) on p. 389. (Their function $F(t,\omega)$ is equal a.s. to $\E(I_{(-\infty,t]}(X_1)\Vert {\mathscr T})=P(X_1\leq t\Vert {\mathscr T})$, which fact follows from the definition of $F(t,\omega)$ on p. 386, line - 9 as the a.s. limit of $(1/n)S_n(I_{(-\infty,t]})$ and the strong law of large numbers (\ref{eq:slln})). The second claim is immediate.
\end{proof}

\section{Predictive PAC learnability} 

\begin{definition}
Let $X_1,X_2,\ldots$ be an exchangeable sequence of random variables with values in a standard Borel space $\Omega$. Denote $P$ the joint distribution on $\Omega^{\infty}$. We say that
a learning rule $\mathcal L$ for a function class $\mathscr F$ on $\Omega$ is {\em predictively PAC} with sample complexity $s(\delta,\e)$ (under the sequence $(X_i)$), if for every $f\in {\mathscr F}$ and each $\e,\delta>0$, whenever $n\geq s(\delta,\e)$, one has
\begin{equation}
P\{\sigma\colon 
\E(\vert{(\mathcal L}(f\upharpoonright \sigma)-f)(X_{n+1})\vert
\Vert X_1,X_2,\ldots, X_n)>\e \} <\delta.
\end{equation}
If $\mathcal P$ is a family of sequences of exchangeable random variables, then
we say that a function class $\mathscr F$ is {\em predictively PAC learnable under $\mathcal P$} 
if it admits a learning rule $\mathcal L$ that is predictively PAC under every exchangeable sequence $(X_i)\in {\mathcal P}$, with the sample complexity uniformly bounded by some function $s(\delta,\e)$. Finally, if $\mathscr F$ is predictively PAC learnable under the family of all exchangeable sequences $(X_i)$, we will simply say that $\mathscr F$ is predictively PAC learnable.
\end{definition}

The following theorem is the main result of the article. 
It allows to deduce predictive PAC learnability from the distribution-free PAC learnability.
The proof bypasses a uniform Glivenko--Cantelli theorem for exchangeable variables.

\begin{theorem}
\label{th:main}
Let $\mathscr F$ be a non-trivial universally separable function class on a standard Borel space $\Omega$ which is uniform Glivenko-Cantelli (in the classical sense), with the sample complexity $n=s(\delta,\e)$.
%, and let $\mathcal L$ be a learning rule. 
% Let $X_1,X_2,\ldots$ be an exchangeable family of random variables with the joint distribution $P$. 
Then $\mathscr F$ is predictive PAC learnable with the sample complexity $s(\delta\e,\e/2)$ under the family of all sequences of $\Omega$-valued exchangeable random variables.
\end{theorem}

\begin{IEEEproof} 
For every $n$, let $\e_n$ be the smallest $\e>0$ with the property $s(0.5,\e)\leq n$. Since $\mathscr F$ is non-trivial, that is, contains at least two functions, $\e_n>0$. 
Let $\mathscr F^\prime$ be a countable dense subfamily of $\mathscr F$ such that every $f\in {\mathscr F}$ is a pointwise limit of a sequence of functions from $\mathscr F^\prime$. For every $\sigma$, the set of samples of the form $f\upharpoonright \sigma$, $f\in {\mathscr F}^\prime$ is clearly dense in the set of samples $f\upharpoonright \sigma$, $f\in {\mathscr F}$. For this reason, using standard selection theorems (e.g. Theorem 5.3.2 in \cite{dudley}), one can construct a measurable emprical risk minimization learning rule $\mathcal L$ on the set of samples 
\[{\mathcal S}_n({\mathscr F}) = \{(f\upharpoonright\sigma)\colon \sigma\in\Omega^n, ~~f\in {\mathscr F}\},\]
taking values in the countable family $\mathscr F^\prime$ and such that for every $n$ and each $(\sigma,s)\in {\mathcal S}_n({\mathscr F})$
\[\frac 1n S_n({\mathcal L}(s)\upharpoonright\sigma - s)<\e_n.\]
Notice that for every $n\geq s(\delta,\e)$, whenever $\delta\leq 0.5$, one has $\e_0\leq \e$, and so $\e+\e_0<2\e$. For this reason, and taking into account the uniform Glivenko-Cantelli property of $\mathscr F$, for every $\theta\in P(\Omega)$ and each $f\in {\mathscr F}$ one has

\begin{equation}
P\left\{ \E_\theta({\mathcal L}(f\upharpoonright \sigma)-f)\geq 2\e \right\}<\delta.
\label{eq:2e}
\end{equation}

Now let $f\in {\mathscr F}$ and $\e,\delta>0$. According to Eq. (\ref{eq:constant}), for a.e. $\theta\in P(\Omega)$ there is a subset $W=W_{\theta}\subseteq\Omega$ with $\theta(W)=1$ and such that for every $\omega\in W$ and each $g\in \{f\}\cup {\mathscr F}^\prime$,
\[\E(g\Vert {\mathscr T})(\omega) = \E_{\theta}(g).\]
Let $\sigma_n(\omega)$ denote, for short, the sequence of values $X_1(\omega),X_2(\omega),\ldots,X_n(\omega)$. Define
\begin{equation}
\label{eq:a}
A = \{\omega\colon \E\left(
\left\vert {\mathcal L}(f\upharpoonright\sigma_n(\omega))(X_1)-f(X_1)\right\vert \Vert {\mathscr T}\right)(\omega)<2\e\}.
\end{equation}
For a.e. $\theta$, one has, $\theta$-a.s.,
\begin{equation}
A \cap W_{\theta}= \{\omega\colon \E_{\theta}\left(
\left\vert {\mathcal L}(f\upharpoonright\sigma_n(\omega))-f\right\vert \right)<2\e\}.
\end{equation}
According to (\ref{eq:2e}), once $n\geq s(\delta,\e)$,
\[\theta(A\cap W_{\theta})\geq 1-\delta,\]
and consequently
\[P(A)=\int\theta(A)\, \eta(d\theta)\geq 1-\delta.\]
Because of symmetry, we can replace $X_1$ in the definition (\ref{eq:a}) of $A$  with $X_{n+1}$.

Now we are applying Lemma \ref{l:tail} to the countable family of functions ${\mathscr G}=\{f\}\cup \{{\mathcal L}(f\upharpoonright\sigma)\colon \sigma\in\Omega^n\}$.
Conditioning on $X_1,X_2,\ldots,X_n$ amounts to integrating with respect to the conditional distribution $P(d\omega\Vert X_1,X_2,\ldots,X_n)$. One must have
\[P\{\omega\colon P(A^c\Vert X_1,X_2,\ldots,X_n)(\omega)\geq 2\e\}< \delta\e^{-1}.\]
We conclude:
\begin{eqnarray*}
P\{\sigma\in\Omega^n\colon 
\E(\vert{\mathcal L}(\sigma,f\vert_\sigma)-f\vert
\Vert X_1,X_2,\ldots, X_n)<2\e \} \\ > 1-\delta\e^{-1}.
\end{eqnarray*}
\end{IEEEproof}

\begin{remark}
The proof can be modified so that $\e/2$ is replaced with $\e-\gamma_n$ for an arbitrarily sequence $\gamma_n\downarrow 0$. We have only chosen $\e/2$ for simplicity.
On the other hand, the extra factor of $\e$ added to $\delta$ does not make much difference, because --- unlike the learning precision $\e$ --- the confidence parameter $\delta$ is well known to be ``cheap''. 
\end{remark}

% \begin{corollary}
% Let $\mathscr F$ be a function class on a standard Borel space $\Omega$. 
% \end{corollary}

\begin{corollary}
Let $\mathscr C$ be a universally separable concept class on a standard Borel space $\Omega$ having finite VC-dimension $d$. Then $\mathscr C$ admits a learning rule which is predictive PAC learnable with regard to any sequence of exchangeable data inputs, with the sample complexity bound
\[s(\delta,\e)= \max\left\{\frac{16d}{\e}\lg\frac{16e}{\e},\frac 8{\e}\lg\frac{2}{\delta}+\frac {8}{\e}\lg\frac{1}{\e} \right\}.
\]
\end{corollary}

The proof follows from Theorem \ref{th:main} and the sample complexity bound for distribution-free PAC learnability (\cite{vidyasagar}, Theorem 7.8),
\[s(\delta,\e) = \max\left\{\frac{8d}{\e}\lg\frac{8e}{\e},\frac 4{\e}\lg\frac{2}{\delta} \right\}.
\]

\section{Conclusion}

Predictable PAC learnability of a function class $\mathscr F$ allows to bound, with high confidence, the probability of misclassification of a value of a classifier function $f\in{\mathscr F}$ at any future data sample $X_i(\omega)$, $i\geq n$, given the values of $f$ on a multisample $X_1(\omega),X_2(\omega),\ldots,X_n(\omega)$. Under this version of learnability, the function $f\in {\mathscr F}$ cannot be learned in general, it is only its future values that can be predicted with high confidence.
For a large number of problems of statistical learning, this is apparently sufficient.

In statistics, exchangeable random variables and de Finetti's theorem are at the forefront of an ongoing discussion between frequentists and bayesians. (Cf. \cite{billingsley}, p. 475.) There is however no need to enter the fray and choose sides, simply because, 
in Vapnik's words \cite{vapnik98}, p. 720, 
\begin{quote}
``Statistical learning theory does not belong to any specific branch of science: It has its own goals, its own paradigm, and its own techniques.

Statisticians (who have their own paradigm) never considered this theory as part of statistics''. 
\end{quote}

Thus, our new approach can be seen just as an addition to the classical framework of learning theory, posessing its own inner dynamics and putting forward a number of open questions. 

Among the most immediate, let us mention the following three, all concerning Theorem \ref{th:main}. Can one maintain the initial sample complexity $s(\delta,\e)$ in the conclusion of the result? Does the theorem hold under less restrictive measurability assumptions on $\mathscr F$ than universal separability, for instance, on an assumption that $\mathscr F$ is image admissible Souslin (\cite{dudley}, pages 186--187)? Can one conclude that $\mathscr F$ is {\em consistently} predictive PAC learnable, that is, predictive PAC learnable under {\em every} consistent learning rule $\mathcal L$? 

\section*{Acknowledgment}

I am indebted to Claus K\"ostler from whose seminar and conference presentations I have learned about exchangeable random variables and de Finetti's theorem.



%\bibliographystyle{IEEEtran}

%\newpage

\end{document}